\documentclass[12pt]{article}

\usepackage{iclr2016_conference,times}
\usepackage{hyperref}
\usepackage{url}
\usepackage{subfigure}
\usepackage{amsfonts}
\usepackage{amsthm}
\usepackage{amsmath}
\usepackage{appendix}
\usepackage{booktabs}
\usepackage{graphicx}

\title{Local minima in training of neural networks}

\author{Grzegorz \'{S}wirszcz, Wojciech Marian Czarnecki \& Razvan Pascanu \\
DeepMind\\
London, UK \\
\texttt{\{swirszcz,lejlot,razp\}@google.com} }

\DeclareMathOperator*{\argmin}{arg\,min}

\usepackage{color}

\newcommand{\themodel}{\mathcal{M}}

\newcommand{\RR}{\mathbb{R}}

\newtheorem{thm}{Theorem}
\newtheorem{prop}{Proposition}

\newtheorem{rem}{Remark}
\newtheorem{lem}{Lemma}
\newtheorem{cor}{Corollary}
\newtheorem{defin}{Definition}
\newtheorem{conj}{Conjecture}

\begin{document}

\maketitle

\begin{abstract}
     Training of a neural network is often formulated as a task of finding a
     ``good'' minimum of an error surface - the graph of the loss expressed as
     a function of its weights. Due to the growing popularity of deep learning,
     the classical problem of studying the error surfaces of neural networks is
     now in the focus of many researchers. This stems from a long standing
     question.  Given that deep networks are highly nonlinear systems optimized
     by local gradient methods, why do they not seem to be affected by
     \textit{bad} local minima?  As much as it is often observed in practice
     that training of deep models using gradient methods works well, little is
     understood about why it happens.  A lot of research efforts has been
     dedicated recently for proving the good behavior of training neural
     networks.  In this paper we adapt the complementary approach of studying
     the possible obstacles. We present several concrete examples of datasets
     which cause the error surface to have a strongly suboptimal local minimum.
\end{abstract}

\section{Introduction}

Deep Learning \citep{lecunNature,schmidhuberDeepLearning} is a fast growing
subfield of machine learning, with many impressive results. Images are being
classified with super-human accuracy (e.g \citet{HeZRS15,SzegedyVISW16}), the
quality of machine translation is reaching new heights (e.g.
\citet{Seq2SeqSutskever14,bahdanau2015,WuSCLNMKCGMKSJL16}). For reinforcement
learning, deep architectures had been successfully used to learn to play Atari
games \citep{mnih2015human, mnih2016asynchronous} or the game of Go
\citep{silver16}.  As always, in case of a fast-progressing domain with
practical application, our theoretical understanding is moving forwards slower
than the fast forefront of empirical success. We are training more and more
complex models, in spite of the fact that this training relies on non-convex
functions which are optimized using local gradient descent methods.  In the
light of these empirical results, many efforts have been made to explain why
the training of deep networks works so well (see the literature review in the
next section). The authors believe, that equally important as it is to try to
understand how and why the training of neural networks behaves well, it is also
to understand what can go wrong.

There are empirical examples of things not always working well. Learning is
susceptible to adversarial examples \citep{NguyenYC15,GoodfellowSS14,Fawzi16}.
Neural networks identifying road stop signs as an interior of a refrigerator,
invisible to human eyes perturbations make a perfectly good model suddenly
misclassify, are some of the well known examples.  Additionally, numerous
``tricks of the trade'', like batch normalization or skip connections, had to
be invented to address slow convergence or poor results of learning.  That is
because, despite the optimism, the ``out of the box'' gradient descent is often
not working well enough.

The approach in this paper is to look for fundamental reasons for training not
behaving well. The goal was to construct as small as possible datasets that
lead to emergence of bad local minima. The error surface is an extremely
complicated mathematical object. The authors believe, that the strategy for
improving our understanding of the structure of error surface is to build a
knowledge base around it. Constructing examples such as the ones presented
serves two main purposed. First, it can help formulate the necessary
assumptions behind theorems showing the convergence of neural network to a good
minima, assumptions that will exclude these datasets.  The second goal is more
practical - to inspire design of better learning algorithms.

\section{Literature Review}
\label{sec:lit_review}

One hypothesis for why learning is well behaved in neural networks is put
forward in \citet{Dauphin13}.  We will refer to it as the ``no bad local
minima'' hypothesis. The key observation of this work is that intuitions from
low-dimensional spaces are usually misleading when moving to high-dimensional
spaces. The work  makes a connection with profound results obtained in statistical
physics. In particular \cite{Fyodorov07,Bray07} showed, using the Replica
Theory \citep{Parisi07}, that random Gaussian error functions have a particular
friendly structure. Namely, if one looks at all the critical points of the
function and plots error versus the (Morse) index of the critical point (the
number of negative eigenvalues of the Hessian), these points align nicely on a
monotonically increasing curve. That is, all points with a low index (note
that every minimum has this index equal to $0$) have roughly the same
performance, while critical points of high error implicitly have a large number
of negative eigenvalue which means they are saddle points.

The claim of \cite{Dauphin13} is that the same structure holds for neural networks
as well, when they become large enough. This provides an appealing conjecture
of not only why learning results in well performing models, but also why it
does so reliably. Similar claim is put forward in \cite{Sagun2014}. These
intuitions can also be traced back to the earlier work of \cite{Baldi89}, which
shows that an MLP with a single \emph{linear} intermediate layer has \emph{no}
local minima, only saddle points and a global minimum. Extensions of these
early results can be found in \cite{Saxe-ICLR2014, Saxe-Cogsci}.

\cite{ChoromanskaHMAL15}  provides a study of the conjecture
that rests on recasting a neural network as a spin-glass model.
To obtain this result several assumptions need to be made, which the authors of
the work, at that time, acknowledged that were not realistic in practice.
The same line of attack is taken by \cite{Kawaguchi16}. Most of these
derivations do not hold in the practical case of finite size datasets
and finite size models.

\cite{Goodfellow16} argues and provides some empirical evidence that while
moving from the original initialization of the model along a straight line to
the solution (found via gradient descent) the loss seems to be only
monotonically decreasing, which speaks towards the apparent convexity of the
problem. \cite{Soudry16, Safran2015} also look at the error surface of the
neural network, providing theoretical arguments for the error surface becoming
well-behaved in the case of overparametrized models.

A different view, presented in \cite{lin2016cheap, Shamir16},
is that the underlying easiness of optimizing deep networks does not
simply rest just in the emerging structures due to high-dimensional spaces, but
is rather tightly connected to the intrinsic characteristics of the data these
models are run on.

\section{Examples of bad local minima}

In this section we present examples of bad local minima. Speaking more
precisely, we present examples of
datasets and architectures such that training using a gradient
descent  can converge to a suboptimal local minumum.

\subsection{Local minima in a sigmoid-based classification}
One of the main goals the authors set for themselves was to show that a sigmoid-based neural network can have a suboptimal ``finite'' local minimum. It turned out that no such example has been widely known to the community, and that there was no agreement to even whether such minimum could exist at all.
\subsubsection{Constructing the example}
By a ``finite'' local minimum we understand a local minimum produced by a set of ``finite'' weights. That is, a minimum that is not caused by some (or all of) sigmoids saturating - trying to become a step-function. In this sense, the minima presented in section~\ref{sec:jelly} are not finite minima.%

The task of constructing this example turned out to be surprisingly difficult. The ability of a sigmoid-based neural network to ``wiggle itself out'' of the most sophisticated traps the authors were creating was both impressive and challenging.

What a few first failed attempts made us realize was, that the nature of a successful example would have to be both geometric and analytic at the same time. What ``deadlocks'' a sigmoid-based neural network is not only the geometric configuration of the points, but also the very precise cross-ratios of distances between them. The successful construction of the presented example was a combination of studying the failed attempts generated by a guesswork and then trying to block the ``escape routes'' with a gradient descent in the data space. Once a ``close enough'' configuration of points was deduced, the gradient descent was applied to the datapoints (in the data space) in order to minimize the length of the gradient in the weights space of the loss function. This procedure modifies the dataset in such a way that the (fixed) set of weights becomes a critical point of the error surface, but starting from a randomly chosen dataset almost surely produces a saddle point, instead of a minimum. Using the ``close enough'' configuration yields a higher chance of finding a true local minimum.%

The authors constructed several examples of local minima for a 2-2-1 (more detailed description below) sigmoid-based neural network, using 16, 14, 12 and 10 datapoints. As of today we know 4 different examples of 10-point datasets that lead to a suboptimal minimum. The authors conjecture this is the minimum amount of points required to ``deadlock'' this architecture, i.e that it is not possible to construct an example using 9 or less points. All the 10-point examples have a geometric configuration resembling a ``figure 8 shape'', as presented in Figure~\ref{fig:sigmoid_deadlock}.
\subsubsection{The example}
Let a dataset $\mathcal{D} = \{ (x_i,y_i) \}_{i=1}^{10}$ be defined as
\begin{eqnarray*}
&&\mathbf{x}_1 = (2.8, 0.4),\; \mathbf{x}_2 = (3.1, 4.3),\; \mathbf{x}_3 = (0.1, -3.4), \; \mathbf{x}_4 = (-4.2, -3.3),\\
&&\mathbf{x}_5 = (-0.5, 0.2),\; \mathbf{x}_6 = (-2.7, -0.4),\; \mathbf{x}_7= (-3., -4.3), \mathbf{x}_8 = (-0.1, 3.4),\\
&&\mathbf{x}_9 =  (4.2, 3.2), \; \mathbf{x}_{10} = ( 0.4, -0.1),
\end{eqnarray*}
$y_1 = \ldots = y_5 =  1$, $y_6 = \ldots = y_{10} =  0$ (see Figure \ref{fig:sigmoid_deadlock}):

\begin{figure}[h]
\centering
{\includegraphics[width=8.5cm]{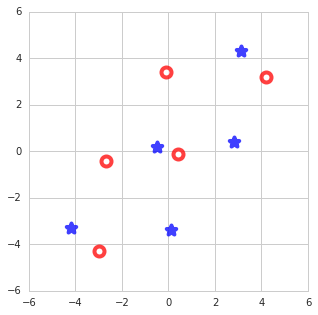}}
\caption{A $10$-point dataset causing the error surface of $2-2-1$
    sigmoid-based Neural Network to have a bad local minimum.}
\label{fig:sigmoid_deadlock}
\end{figure}

\noindent Let us consider a neural network (using notation $\sigma(x) = (1+\exp(-x))^{-1}$)
\begin{equation}
\label{eqn::nn}
\themodel((x_0, x_1)) = \sigma( v_{0} \sigma(w_{0,0} x_0 + w_{0,1} x_1 + b_0) + v_{1} \sigma(w_{1,0} x_1 + w_{1,1} x_2 + b_1) + c)
\end{equation}
and let us use the standard negative log-likelihood cross-entropy loss
\[
\mathcal{L}(v_0, w_{0,0},w_{0,1},b_0v_1,w_{1,0},w_{1,1},b_1,c) = - \sum\limits_{i=1}^{10} y_i \log(\themodel(\mathbf{x}_i)) + (1-y_i) \log(1 - \themodel(\mathbf{x}_i)).
\]

\noindent Then there holds
\begin{thm}
The point $\hat{W}$ in the (9-dimensional) weight space consisting of
\begin{eqnarray*}
&& w_{0,0} = 1.05954587, \;  w_{0,1} = -0.05625762,\\
&& w_{1,0} = -0.03749863, \; w_{1,1} = 1.09518945\\
&& b_0 = -0.050686, b_1 = -0.06894291,\\
&& v_0 = 3.76921058,\; v_1 = -3.72139955,\; c = -0.0148436.
\end{eqnarray*}
    is a local minimum of the error surface of the neural network~(\ref{eqn::nn})
    with negative log-loss value 0.577738 (corresponding to likelihood =
    0.561166). This local minimum has accuracy $0.4$. This point is not a
    global minimum.
\end{thm}

\begin{proof}
The gradient of $\themodel$ at $\hat{W} = \mathbf{0}$ and eigenvalues of the Hessian are
\[
\begin{array}{l}
0.0007787149706922058671933702882302\\
0.09566127257833993223676197073566\\
0.1737731623214082676475029319782\\
0.22063866543084709867511532964466\\
0.4155934900503221206301551236848\\
0.9246044147479949855459498868096\\
3.803556801786189964831977345844\\
4.572940690876952005351090155283\\
6.391098807223509384191737359951
\end{array}
\]
so by standard theorems from Analysis the point $\hat{W}$ is a local minimum.\\
At the point $W_{0}$
\begin{eqnarray*}
&& w_{0,0} = 5.67526388,\;  w_{0,0} = 0.50532424\\
&& w_{1,0} = -68.69289398,\; w_{1,1} =  -5.17422295\\
&& b_0 = 3.23905253,\; b_1 =   0.24047163\\
&& v_0 = -44.49337769,\; v_1 =  45.87974167,\; c = -0.69310206
\end{eqnarray*}
    the performance of the model is: accuracy $= 0.8$, loss $= 0.381913$,
    likelihood $= 0.682555$, therefore $\hat{W}$ is not a global minimum.
\end{proof}

\begin{rem}
    It is worth noting that we are not claiming that $W_0$ is a global minimum. In
    fact it is not even a minimum at all. Note the very high values of $v_0$
    and $v_1$. It happens because the final sigmoid is struggling to
    approximate a step function. It is a common phenomenon in training neural
    networks, the training does not converge to a minimum, but gets stopped
    while trying to converge to a point at infinity.
\end{rem}

\begin{rem}
The point
\begin{eqnarray*}
&& w_{0,0} = 22.3641243,\;  w_{0,1} = -12.53928375,\\
&& w_{1,0} = -44.85849762,\; w_{1,1} -3.51257443\\
&& b_0 = -35.75595093,\; b_1 = -23.58968163\\
&& v_0 = 15.43178844,\; v_1-15.02632332,\; c = -0.40546528
\end{eqnarray*}
with loss $= 0.475135$,  likelihood $= 0.621801$ accuracy $= 0.7$ is yet
    another suboptimal point of convergence of the training.
\end{rem}

\subsection{Local minima in a rectifier-based regression}

Rectifier-based models are the de facto standard in most applications of
neural networks. In this section  we present 3 examples of local minima for
regression using a single layer with 1, 2 and 3 hidden rectifier units on
$1$-dimensional data (see  Figure~\ref{fig:example_minima}).

\begin{rem}
For the ReLU-s, the activation function is simply the $\max$ between
$0$ and the linear projection of the input. Hence, it has two modes of operation,
it is either in the linear regime or the saturated regime. Obviously,  no gradient flows through a
saturated unit, hence a particular simple mechanism for locking a network in a suboptimal
solution is to have a subset of datapoints such that all units (e.g. on a given layer) are
saturated, and there is no gradient for fitting those points. We will refer to such
points as being in the {\rm blind spot} of the model and explore this phenomenon more properly
in section \ref{sec:blind_spot}.
\end{rem}

\begin{rem}
The examples presented in this section go beyond relying solely on the of blind-spots of the model.
\end{rem}

For the sake of simplicity of our presentation we will describe in detail the
case with 1 hidden neuron, the other two cases can be treated similarly. In
case of one hidden neuron the regression problem becomes
\begin{equation}
\label{eqn::ReLU_regression_1}
\argmin\limits_{w, b, v, c} \mathcal{L}(w, b, v, c) = \sum_{i=1}^{n} \left( v \cdot \mathrm{ReLU}(w {x}_i + b) + c - {y}_i\right)^2.
\end{equation}

\noindent Consider a dataset $\mathcal{D}_1$ (see Figure \ref{fig:example_minima}~(a)):
\[
({x}_1, y_1) = (5,2), ({x}_2, y_2) = (4,1), ({x}_3, y_3) = (3,0), ({x}_4, y_4) = (1,-3), ({x}_5, y_5) = (-1,3).
\]

\begin{prop}
\label{prop:particularCase}
For the dataset $\mathcal{D}_1$ and $\mathcal{L}$ defined in Equation (\ref{eqn::ReLU_regression_1}) the point $v = 1, b = -3, w = 1, c=0$ is a local minimum of $\mathcal{L}$, which is not a global minimum.
\end{prop}

\begin{proof}
There holds $\mathcal{L}(1,-3,1,0) = 0+0+0+9+9 = 18$, and $\mathcal{L}(-7,-4,1,0) = 4 + 1 + 0 + 9 + 0 = 14$, thus $(1,-3,1,0)$ cannot be a global minimum. It remains to prove that $(1,-3,1,0)$ is a local minimum, i.e. that $\mathcal{L}(1+\delta_w,-3+\delta_b,1+\delta_v,\delta_c) \ge \mathcal{L}(1,-3,1,0)$ for $|\delta_w|$, $|\delta_b|$, $|\delta_v|$, $|\delta_c|$ sufficiently small. We need to consider two cases:\\
{\bf ReLU activated at } ${x}_3$. In that case
\[
\mathcal{L}(1+\delta_w,-3+\delta_b,1+\delta_v,\delta_c) =
\]
\[
\left((1+\delta_v)(3 + 3 \delta_w - 3 + \delta_b) + \delta_c\right)^2 + \left((1+\delta_v)(4 + 4 \delta_w - 3 + \delta_b) + \delta_c -1\right)^2 +
\]
\[
\left((1+\delta_v)(5 + 5 \delta_w - 3 + \delta_b) + \delta_c -2\right)^2 + (\delta_c+3)^2+(\delta_c-3)^2.
\]
We introduce new variables $x = (\delta_w+1)(1+\delta_v)-1$, $y = (\delta_b-3)(1+\delta_v) + 3$, $z = \delta_c$. The formula becomes
\[
\left(3 x + y + z\right)^2 + \left(4 x + y + z\right)^2 +
\left(5 x + y + z\right)^2 + 2 z^2 + 18 \ge 18,
\]
which ends the proof in this case.\\
{\bf ReLU deactivated at } ${x}_3$. In that case
\[
\mathcal{L}(1+\delta_w,-3+\delta_b,1+\delta_v,\delta_c) = \delta_c^2 + \left((1+\delta_v)(4 + 4 \delta_w - 3 + \delta_b) + \delta_c -1\right)^2 +
\]
\[
\left((1+\delta_v)(5 + 5 \delta_w - 3 + \delta_b) + \delta_c -2\right)^2 + (\delta_c+3)^2+(\delta_c-3)^2 =
\]
\[
\left(4 x + y + z\right)^2 +
\left(5 x + y + z\right)^2 + 3 z^2 + 18 \ge 18
\]
(we used $x = (\delta_w+1)(1+\delta_v)-1$, $y = (\delta_b-3)(1+\delta_v) + 3$, $z = \delta_c$ again).\\
Note that due to the assumption that $|\delta_w|$, $|\delta_b|$, $|\delta_v|$, $|\delta_c|$ are sufficiently small the ReLU is always activated at ${x}_1$, ${x}_2$ and deactivated at ${x}_4,{x}_5$.
\end{proof}

\begin{rem}
The point $(1,-3,1,0)$ is a minimum, but it is not a ``strict'' minimum - it is
    not isolated, but lies on a 1-dimensional manifold at which $\mathcal{L}
    \equiv 18$ instead.
\end{rem}

\begin{rem} The following examples show that blind spots are not the only reason a model can be stuck in a suboptimal solution. Even more surprisingly, they also show that the blind spots can be completely absent in the local optima, while at the same time being present in the global solution.
\end{rem}

\begin{figure}[t]
    \centering

    a) \includegraphics[width=7.5cm]{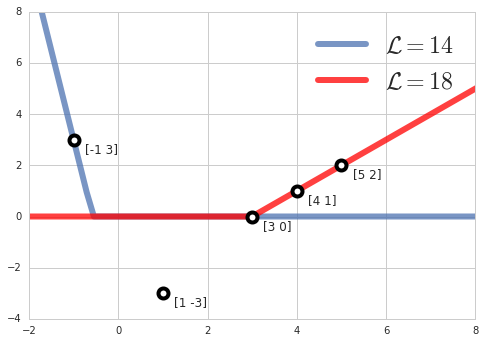}\\
    b) \includegraphics[width=7.5cm]{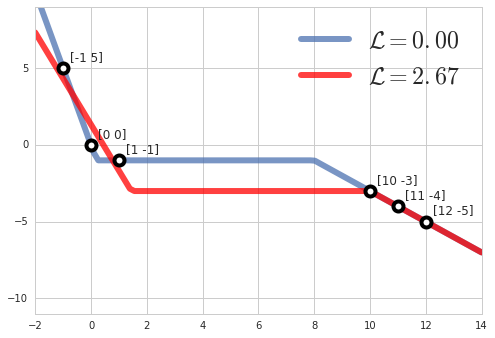}\\
    c) \includegraphics[width=7.5cm]{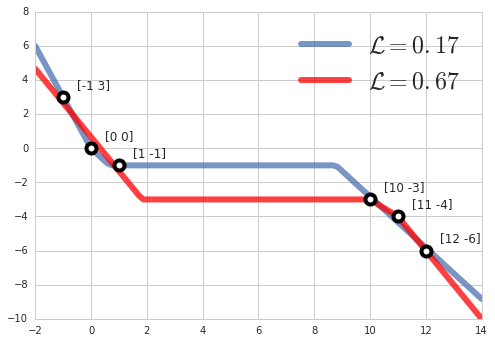}

\caption{Local minima for ReLU-based regression. Both lines represent local optima, where the blue one is better than the red one. a) 1 hidden neuron b) 2 hidden neurons c) 3 hidden neurons.}
\label{fig:example_minima}
\end{figure}

\begin{prop}
Let us consider a dataset $\mathcal{D}_2$ with $d=1$, given by points
    $(x_1,y_1)=(-1,5), (x_2,y_2)=(0, 0), (x_3,y_3)=(1,-1), (x_4,y_4)=(10,-3),
    (x_5,y_5)=(11,-4), (x_6,y_6)=(12,-5)$ (Figure \ref{fig:example_minima}
    (b)). Then, for a rectifier network with $m=2$ hidden units and a squared
    error loss the set of weights $\mathbf{w}=(-5, -1), \mathbf{b}=(1, -8),
    \mathbf{v}=(1,-1), c=-1$ is a global minimum (with perfect fit) and the set
    of weights $\mathbf{w}=(-3, -1), \mathbf{b}=(4+\tfrac{1}{3}, -10),
    \mathbf{v}=(1,-1), c=-3$ is a suboptimal local minimum.
\end{prop}
\begin{proof}
Analogous to the previous one.
\end{proof}

\noindent Maybe surprisingly, the global solution has a \emph{blind spot} - all neurons
deactivate in ${x}_3$. Nevertheless, the network still has a $0$ training
error.

\begin{prop}
\label{prop:particularCase3}
Let us consider a dataset $\mathcal{D}_3$ with $d=1$, given by points $(x_1,
    y_1)=(-1,3), (x_2, y_2)=(0, 0), (x_3,y_3)=(1,-1), (x_4,y_4)=(10,-3),
    (x_5,y_5)=(11,-4),(x_6,y_6)=(12,-6)$ (Figure \ref{fig:example_minima} (c)).
    Then, for a rectifier network with $m=3$ hidden units and a squared error
    loss the set of weights $\mathbf{w}=(-1.5, -1.5, 1.5), \mathbf{b}=(1, 0,
    -13 - \tfrac{1}{6}), \mathbf{v}=(1,1,-1), c=-1$ is a better local minimum
    than the local minimum obtained for $\mathbf{w}=(-2, 1,1),
    \mathbf{b}=(3+\tfrac{2}{3}, -10, -11), \mathbf{v}=(1,-1,-1), c=-3$.
\end{prop}
\begin{proof}
Completely analogous, using the fact that in each part of the space linear
    models are either optimal linear regression fits (if there is just one
    neuron active) or perfect (0 error) fit when two neurons are active and
    combined.
\end{proof}

Note that again that the above construction is not relying on the \emph{blind
spot} phenomenon. The idea behind this example is that if, due to initial
conditions, the model partitions the input space in a suboptimal way, it might
become impossible to find the optimal partitioning using gradient descent. Let
us call $(-\infty, 6)$ the region I, and $[6, \infty)$ region II.
Both solutions in Proposition~\ref{prop:particularCase3} are constructed in
such way that each one has the best
fit for the points assigned to any given region, the only difference being the
number of hidden units used to describe each of them.
In the local optimum two neurons are used to describe region II, while only one
describes region I. Symmetrically, the better solution assigns two neurons to
region I (which is more complex) and only one to region II.

\begin{conj}
We conjecture that the core idea behind this construction can be generalized
    (in a non-trivial way) to high-dimensional problems.
\end{conj}

\subsection{The flattened XOR - suboptimal models in classification using ReLu and sigmoids}
\label{sec:jelly}
In this section we look at a slight variation on one of the most theoretically well-studied datasets, the XOR problem. By exploiting observations made in the failure modes observed for the XOR problem, we were able to construct a similar dataset, the ``flattened XOR'', that results in suboptimal learning dynamics. The dataset is formed of four datapoints, where the positive class is given by  $(1.0, 0.0), (0.2, 0.6)$ and the negative one by $(0.0, 1.0), (0.6, 0.2)$, see Figure~\ref{fig:jellyfish_dataset}.
\begin{figure}[h]
\centering
{\includegraphics[width=7.5cm]{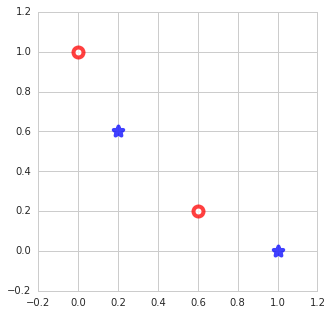}}
\caption{The ``flattened XOR'' dataset.}
\label{fig:jellyfish_dataset}
\end{figure}
We analyze the dataset using a single hidden layer network (with either ReLU or sigmoid units).

A first observation is that while SGD can solve the task with only 2 hidden units,
full batch methods do not always succeed. Replacing gradient descent with more aggressive optimizers like Adam does not seem to help, but rather tends to make it more likely to get stuck in suboptimal solutions (Table~\ref{tab:small_convergence}).

\begin{figure}[h]
\centering
a) \includegraphics[width=10.5cm]{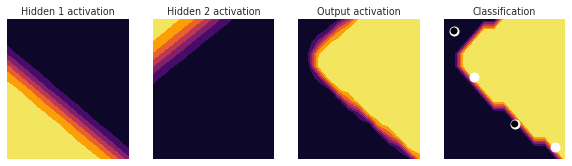}
\\
b) \includegraphics[width=10.5cm]{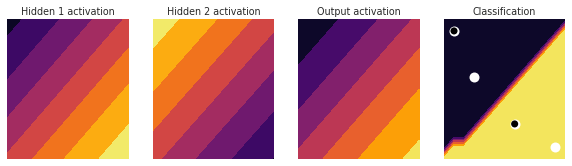}
\caption{Examples of different outcomes of learning on the flattened XOR dataset. a) Optimally converged net for flattened XOR. b) Stuck net for flattened XOR.}
\label{fig:jellyfish}
\end{figure}

\begin{table}
\center
\begin{tabular}{l|rrrrr|rrrrr}
    \toprule
   $h$ &   & XOR & XOR & fXOR & fXOR
   &   & XOR & XOR & fXOR & fXOR\\
    &   & ReLU & Sigmoid & ReLU & Sigmoid &
   & ReLU & Sigmoid & ReLU & Sigmoid\\

    \midrule
   2& Adam  & 28\% & 79\%   & 7\%       & 0\%   &
   GD&   23\% & 90\%   & 16\%       & 62\%       \\
   3&Adam  & 52\% & 98\%   & 34\%      & 0\%        &
   GD  & 47\% & 100\%   & 33\%      & 100\%      \\
   4&Adam  & 68\% & 100\%  & 50\%      & 2\%        &
   GD  & 70\% & 100\%  & 66\%      & 100\%        \\
   5&Adam  & 81\% & 100\%  & 51\%      & 27\%       &
   GD  & 80\% & 100\%  & 68\%      & 100\%      \\
   6&Adam  & 91\% & 100\%  & 61\%      & 17\%       &
   GD  & 89\% & 100\%  & 69\%      & 100\%     \\
   7&Adam  & 97\% & 100\%  & 69\%      & 58\%       &
   GD  & 89\% & 100\%  & 86\%      & 100\%      \\
     \bottomrule
\end{tabular}
\caption{``Convergence'' rate for 2-$h$-1 network with random initializations on simple 2-dimensional datasets using either Adam or Gradient Descent (GD) as an optimizer. Comparison between ``regular'' XOR and fXOR - the flattened XOR}
\label{tab:small_convergence}
\end{table}

Compared to the XOR problem it seems the flattened XOR problem poses even more issues, especially for ReLU units, where with 4 hidden units one still only gets 2 out of 3 runs to end with 0 training error (when using GD). One particular observation (see Figure~\ref{fig:jellyfish})  is that in contrast with good solutions, when the model fails on this dataset, its behaviour close to the datapoints is almost linear. We argue hence, that the failure mode might come from having most datapoints concentrated in the same linear region of the model (in ReLU case), hence forcing the model to suboptimally fit these points.

\begin{rem}
In the examples we used ReLU and sigmoid activation functions, as they are the most common used in practice. The similar examples can be constructed for different activation functions, however the constructions need some modifications and get more technically complicated.
\end{rem}

\section{Bad initialization}
\label{sec:blind_spot}
In this section we prove formally a seemingly obvious, but often overlooked fact that for any regression dataset a rectifier model has at least one local minimum. The construction relies on the fact that the dataset is finite. As such, it is bounded, and one can compute conditions for the weights of any given layer of the model such that for any datapoint all the units of that layer are saturated. Furthermore, we show that one can obtain a better solution than the one reached from such a state. The formalization of this result is as follows.

We consider a $k$-layer deep regression model using $m$ ReLU units $\mathrm{ReLU}(x) = \max(0,x)$.  Our dataset is a collection $(\mathbf{x}_i,y_i) \in \RR^d \times \RR$, $i=1,\ldots,N$.
We denote $\mathbf{h}_n(\mathbf{x}_i) = \mathrm{ReLU}(\mathbf{W}_n \mathbf{h}_{n-1}(\mathbf{x}_i) + \mathbf{b}_n)$ where the the $\mathrm{ReLU}$ functions are applied component-wise to the vector $\mathbf{W}_n \mathbf{h}_{n-1}(\mathbf{x}_i)$ and $\mathbf{h}_0(\mathbf{x}_i) = \mathbf{x}_i$.
We also denote the final output of the model by $\themodel(\mathbf{x}_i) = \mathbf{W}_{k}\mathbf{h}_{k-1} + \mathbf{b}_{k}$.
Solving the regression problem means finding
\begin{equation}
\label{eqn::ReLU_regression_general}
\argmin\limits_{
(\mathbf{W}_n)_{n=1}^k,
(\mathbf{b}_n)_{n=1}^k}
\mathcal{L}((\mathbf{W}_n)_{n=1}^k, (\mathbf{b}_n)_{n=1}^k)
 = \sum_{i=1}^{N} \left[\themodel(\mathbf{x}_i) - {y}_i\right]^2.
\end{equation}
Let us state two simple yet in our opinion useful Lemmata.
\begin{lem}[Constant input]
If $\mathbf{x}_1 = \ldots = \mathbf{x}_N$, then the solution to regression~(\ref{eqn::ReLU_regression_general}) has a constant output $\themodel \equiv \frac{y_1+\ldots+y_N}{N}$ (the mean of the values in data).
\end{lem}
\begin{proof}
Obvious from the definitions and the fact, that $\frac{y_1+\ldots+y_N}{N} = \argmin\limits_{c} \sum_{i=1}^{N}(c-y_i)^2$.
\end{proof}
\begin{lem}
If there holds $\mathbf{W}_1 \mathbf{x}_i < - \mathbf{b}_1$ for all $i$-s, then the model $\themodel$ has a constant output. Moreover, applying local optimization does not change the values of $\mathbf{W}_1$, $\mathbf{b}_1$.
\end{lem}
\begin{proof}
Straightforward from the definitions.
\end{proof}
Combining these two lemmata yields:
\begin{cor}
\label{cor::the_only_corollary}
If for any $1 \le j \le k$ there holds $ \mathbf{W}_n  \mathbf{h}_{n-1} < - b_n$ for all $i$-s then, after the training, the model $\themodel$ will output $ \frac{y_1+\ldots+y_N}{N}$.
\end{cor}
We will denote $M(\{ a_1,\ldots,a_L \}) = \frac{a_1 + \ldots + a_L}{L}$ the mean of the numbers $a_1, \ldots, a_L$.
\begin{defin}
We say that the dataset $(\mathbf{x}_i, y_i)$ is \textbf{decent} if there exists $r$ such that $M(\{y_p : \mathbf{x}_p = \mathbf{x}_r\} \neq M(\{y_p : p=1, \ldots, N \})$.
\end{defin}
\begin{thm}
\label{thm:localminima}
Let ${\boldsymbol{\theta}} = ((\mathbf{W}_n)_{n=1}^k, (\mathbf{b}_n)_{n=1}^k)$ be any point in the parameter space satisfying $\mathbf{W}_n \mathbf{h}_n(\mathbf{x}_i) < - \mathbf{b}_n$ (coordinate-wise) for all $i$-s. Then
\begin{itemize}
\item[ i)]${\boldsymbol{\theta}}$ is a local minimum of the error surface,
\item[ii)] if the first layer contains at least 3 neurons and if the dataset $(\mathbf{x}_i, y_i)$ is decent, then ${\boldsymbol{\theta}}$ is not a global minimum.
\end{itemize}
\end{thm}

\begin{proof}
\label{proof:localminima}
Claim i) is a direct consequence of Corollary~\ref{cor::the_only_corollary}. It remains to prove ii). For that it is sufficient to show an example of a set of weighs $\hat{{\boldsymbol{\theta}}} = ((\hat{\mathbf{W}}_n)_{n=1}^k, (\hat{\mathbf{b}}_n)_{n=1}^k)$ such that $\mathcal{L}((\mathbf{W}_n)_{n=1}^k, (\mathbf{b}_n)_{n=1}^k) > \mathcal{L}((\hat{\mathbf{W}}_n)_{n=1}^k, (\hat{\mathbf{b}}_n)_{n=1}^k)$.
Let $r$ be such  that $M(\{y_p : \mathbf{x}_p = \mathbf{x}_r\}) \neq M(\{y_p : p=1, \ldots, N\})$. Such point exists by assumption that the dataset is decent. Let $\mathcal{H}$ be a hyperplane passing through $\mathbf{x}_r$ such that none of the points $\mathbf{x}_s \neq \mathbf{x}_r$ lies on $\mathcal{H}$. Then there exists a vector $\mathbf{v}$ such that $|\mathbf{v}^T (\mathbf{x}_s-\mathbf{x}_r)| > 2$ for all $\mathbf{x}_s \neq \mathbf{x}_r$. Let $\gamma = \mathbf{v}^T \mathbf{x}_r$. 
We define $\mathbf{W}_1$ in such a way that the first row of $\mathbf{W}_1$ is $\mathbf{v}$
, the second row is $2 \mathbf{v}$ and the third one is $\mathbf{v}$ again, and if the first layer has more than 3 neurons, we put all the remaining rows of $\mathbf{W}_1$ to be equal zero. We choose the first three biases of $\mathbf{b}_1$ to be $-\gamma+1$, $-2\gamma$ and $-\gamma -1$ respectively. We denote $\mu = M(\{y_p : \mathbf{x}_p \neq \mathbf{x}_r\})$ and $\nu = M(\{y_p : \mathbf{x}_p = \mathbf{x}_r\})$.  We then choose $\mathbf{W}_2$ to be a matrix whose first row is $(\nu - \mu , \mu -\nu, \nu-\mu, 0, \ldots, 0)$ and the other rows are equal to $0$. Finally, we choose the bias vector $\mathbf{b}_2 = (\mu, 0, \ldots, 0)^T$.\\
If our network has only one layer the output is
$$
(\nu - \mu) \mathrm{ReLU}(\mathbf{v}^T \mathbf{x}_p - \gamma + 1) - (\nu - \mu) \mathrm{ReLU}(2 \mathbf{v}^T \mathbf{x}_p - 2 \gamma) + (\nu - \mu) \mathrm{ReLU}(\mathbf{v}^T \mathbf{x}_p - \gamma - 1) + \mu.
$$
For every $\mathbf{x}_p = \mathbf{x}_r$ this yields $(\nu - \mu) \cdot 1- 0 + 0 + \mu  = \nu$. For any $\mathbf{x}_p \neq \mathbf{x}_r$ we either have $\mathbf{v}^T \mathbf{x}_p - \gamma < -2$ yielding  $0 - 0 + 0 + \mu  = \mu$ or $\mathbf{v}^T \mathbf{x}_p - \gamma > 2$ yielding $(\nu - \mu) (\mathbf{v}^T \mathbf{x}_p - \gamma + 1) - (\nu - \mu) (2 \mathbf{v}^T \mathbf{x}_p - 2 \gamma) + (\nu - \mu) (\mathbf{v}^T \mathbf{x}_p - \gamma - 1) + \mu = \mu$.\\
In case the network has more than 1 hidden layer we set all $\mathbf{W}_n = \mathbf{I}$ (identity matrix) and $\mathbf{b}_n = \mathbf{0}$ for $n = 3, \ldots, k$.\\
If we denote $\bar{\mu} = M(\{ y_p:p= 1, \ldots , N \})$ (mean of all labels), we get:
\begin{equation*}
\begin{aligned}
\mathcal{L}((\hat{\mathbf{W}}_n)_{n=1}^k, (\hat{\mathbf{b}}_n)_{n=1}^k) &= \sum_{\mathbf{x}_p \neq \mathbf{x}_r} (y_i - \mu)^2 + \sum_{\mathbf{x}_p = \mathbf{x}_r} (y_i - \nu)^2 < \\
\sum_{\mathbf{x}_p \neq \mathbf{x}_r} (y_i - \bar{\mu})^2 + \sum_{\mathbf{x}_p = \mathbf{x}_r} (y_i - \bar{\mu})^2 &= \sum_{y_i} (y_i - \bar{\mu})^2 = \mathcal{L}((\mathbf{W}_n)_{n=1}^k, (\mathbf{b}_n)_{n=1}^k).
\end{aligned}
\end{equation*}

We used the fact that for any finite set $A$ the value $M(A)$ is a strict minimum of $f(c) = \sum_{a \in A} (a-c)^2$ and the assumption that $\nu \neq \bar{\mu}$.
\end{proof}

\section{Discussion}
\label{sec:discussion}

Previous results \citep{Dauphin13, Saxe-ICLR2014, ChoromanskaHMAL15} provide
insightful description of the error surface of deep models under general
assumptions divorced from the specifics of the architecture or data. While such
analysis is very valuable not only for building up the intuition but also for
the development of the tools for studying neural networks, it only provides one
facade of the problem.
In this work focused on constructing scenarios in which learning fails, in the hope that
they will help setting up right assumptions for convergence theorems of neural
networks in practical scenarios.

Similar to \cite{lin2016cheap} we put forward a hypothesis that the learning is
only well behaved conditioned on the structure of the data. Understanding of the structure of the error surface
is an extremely challenging problem. We believe that as such, in agreement with
a scientific tradition, it should be approached by gradually building up a
related knowledge base, both by trying to obtain positive results (possibly
under weakened assumptions, as it was done so far) and by studying the
obstacles and limitations arising in concrete examples.

\subsubsection*{Acknowledgments}
We would want to thank Neil Rabinowitz for insightful discussions.

\bibliography{iclr2016_conference}

\end{document}